\documentclass{article}
\usepackage{ijcai23}
\usepackage{times}
\usepackage{soul}
\usepackage{url}
\usepackage[hidelinks]{hyperref}
\usepackage[utf8]{inputenc}
\usepackage[small]{caption}
\usepackage{graphicx}
\usepackage{amsmath}
\usepackage{amsthm}
\usepackage{booktabs}
\usepackage{algorithm}
\usepackage{algorithmic}
\usepackage[switch]{lineno}
\usepackage{float}
\usepackage{amssymb}
\usepackage{multirow}
\usepackage{multicol}

\newtheorem{theorem}{Theorem}
\newtheorem{definition}{Definition}
\newtheorem{proposition}{Proposition}
\newtheorem{lemma}{Lemma}

\title{Differentially Private Graph Neural Network with \\ Importance-Grained Noise Adaption}

\author{
Yuxin Qi$^1$
\and
Xi Lin$^2$
\and
Jun Wu$^{3}$
\affiliations
$^{1,2,3}$Shanghai Jiao Tong University
\emails
\{qiyuxin98, linxi234, junwuhn\}@sjtu.edu.cn
}

\begin{document}

\maketitle

\begin{abstract}
    Graph Neural Networks (GNNs) with differential privacy have been proposed to preserve graph privacy when nodes represent personal and sensitive information. However, the existing methods ignore that nodes with different importance may yield diverse privacy demands, which may lead to over-protect some nodes and decrease model utility. In this paper, we study the problem of importance-grained privacy, where nodes contain personal data that need to be kept private but are critical for training a GNN. We propose NAP-GNN, a node-importance-grained privacy-preserving GNN algorithm with privacy guarantees based on adaptive differential privacy to safeguard node information. First, we propose a Topology-based Node Importance Estimation (TNIE) method to infer unknown node importance with neighborhood and centrality awareness.  Second, an adaptive private aggregation method is proposed to perturb neighborhood aggregation from node-importance-grain. Third, we propose to privately train a graph learning algorithm on perturbed aggregations in adaptive residual connection mode over multi-layers convolution for node-wise tasks. Theoretically analysis shows that NAP-GNN satisfies privacy guarantees. Empirical experiments over real-world graph datasets show that NAP-GNN achieves a better trade-off between privacy and accuracy.
\end{abstract}

\section{Introduction}
In recent years, graph neural networks have achieved outstanding performance in several domains, such as social analysis \cite{yang2021consisrec}, financial anomaly detection \cite{chen2020phishing}, time series analysis \cite{wang2021hierarchical}, and molecule synthesis \cite{gasteiger2021directional}. Through aggregating the feature information of neighboring nodes and fully mining and fusing the topological associations in graph data, graph neural networks yield state-of-art results in tasks such as link prediction \cite{zhao2021csgnn}, node classification \cite{guan2022robognn}, and sub-graph classification \cite{yang2021heterogeneous}. However, graph data in the real world usually contain private information. For example, in social network graphs, the edges between nodes indicate the existence of social attributes, such as being friends. Node features also carry sensitive information, such as the average online time of users per week and the average number of comments per week. To satisfy the privacy preservation demands of nodes, privacy-preserving GNN models have been proposed in \cite{olatunji2021releasing}. The Locally Private Graph Neural Network (LPGNN) presented in \cite{sajadmanesh2021locally} realizes the GNN model framework under local differential privacy guarantee. The features and labels of nodes can be protected. An edge-level differential privacy graph neural network (DP GNN) algorithm was proposed in \cite{wu2022linkteller} by perturbing the adjacency matrix of the graph. Some other researches attempt to address the privacy problem in GNN through federated \cite{jiang2020federated} and split learning \cite{zhou2020privacy}.\par
Although the existing DP GNN models enable privacy data protection, it does not consider the diverse privacy demand of nodes in the graph. Take node degree, for example. Nodes with different degree may yield different privacy protection needs. Most realistic graph data meet the power-law distribution \cite{clauset2009power}. The degree of most nodes is small, and few nodes have a relatively large degree. 
In the social media graph, suppose the nodes represent users, and the edges represent the following behaviors among users. The in-degree of the nodes indicates how many followers they have, and the out-degree indicates how many other users they follow. Users with a larger number of followers have stronger social effects, and other users will notice their comments on social media and their friends list. Nodes with higher degrees require a more vital level of privacy protection. \par
However, existing works treat all nodes' privacy demands as the same and do not provide corresponding privacy guarantees for nodes with different privacy needs, which would lead to over-protecting some nodes, injecting over-needed noise, and decreasing model utility. Another drawback of existing methods protecting GNN through DP-SGD where noise is added on the parameter gradient is that the magnitude of injected noise and the privacy budget is accumulated during the training phase in proportion to epoch number.\par
To solve these problems, we propose a Node-Importance-Grained Adaptive Differential Private Graph Neural Network (NAP-GNN), which adds different extents of noise based on node importance and independent of the training epoch. Our goal is to develop a fine-grained and adaptive differential privacy algorithm in GNN that flexibly distributes privacy budget according to known node importance and uses topology information in graphs.
In this paper, we calculate the degree of a node as the sum of in-degree and out-degree. A supervised node importance estimation method, TNIE, is proposed to capture relations node neighbors and flexibly make centrality awareness adjustments. Then depending on the calculated node importance, we propose an importance-awareness fine-grained permutation method to realize node diverse privacy-preserving demand. An adaptive message passing method with residual connection is applied to improve model utility. We show theoretically that NAP-GNN satisfies the requirement of differential privacy. Experiments on real graph datasets show that the adaptive privacy budget allocation mechanism and adaptive residual aggregation of NAP-GNN can improve the utility of the model and outperforms existing approaches under a given privacy budget. The main contributions of this work are as follows.\par

\begin{itemize}
    \item We study a novel problem of node-importance-related privacy concerns in GNNs. To the best of our knowledge, this is the first work to analyze this issue.
    \item We propose a new algorithm named Node-Importance-Grained Adaptive Differential Private Graph Neural Network (NAP-GNN), which tackles different privacy requirements among nodes by assigning global budget.
    \item We present Topology-based Node Importance Estimation (TNIE), a supervised estimation method, to address unknown nodes' importance estimation problem considering neighborhood and centrality awareness from the topology perspective.
    \item The experimental results on four benchmark graph datasets demonstrate the availability and effectiveness of our proposed NAP-GNN method.
\end{itemize}

\section{Problem Formulation}
In this section, we first revisit the definition of differential privacy proposed by \cite{dwork2014algorithmic} and graph neural work. Then we define the problem of learning a GNN with node data privacy concerns related to node importance.
\subsection{Graph Neural Network}
Let $\mathcal{G}=\{V, E, A, X, Y\}$ be an unweighted undirected graph, where $V$ and $E$ denote the nodes set and edges set. The adjacency matrix $A \in \{0, 1\}^{N \times N}$ represents the link among edges, $|N|$ denotes the node number. For $\forall v_i, v_j \in V$, if there exists an edge between $v_i, v_j$, then $A_{ij} = 1$, for else $A_{ij} = 0$. Node feature of $v \in V$ is a $d$-dimension vector, and the N $\times$ $d$ matrix $X$ represents the stack of all nodes' feature, where $X_v \in X$ denotes the feature of $v$. $Y \in \{0,1\}^{N \times M}$ represents the label of nodes, and M is the class number.  \par
The typical message-passing-based GNN consists of two phases: message aggregation and updating. In the message aggregation phase of $i$-th layer, every node share and receive neighbors embedding of the former $i$-1-th layer and output a new embedding after applying a transformation, which can be defined as:
\begin{equation}
    E_v^i = f_{agg}(\{h_u^{i-1}, u \in \mathcal{N}(v)\}).
\end{equation}
$\mathcal{N}(v)$ denotes the adjacent node set of node $v$, and $h_u^{i-1}$ represents the embedding output of node $u$ at $i$-1-th layer. $f_{agg}$ is the aggregation linear function like SUM, MEAN, MAX etc. $E_v^i$ is the aggregate output of $i$-th layer after the aggregation transformation of all adjacent nodes. Update transformation is employed on the $E_v^i$, which is shown as:
\begin{equation}
    h_v^i = f_{upd}(E_v^i, h_v^{i-1}; \theta_i).
\end{equation}
$f_{upd}$ denotes the learnable function that takes the aggregate vector as input and outputs the embedding of node $v$ at $i$-th layer. $f_{upd}$ is determined by parameter $\theta_i$. The input $h_v^0$ of GNN's first layer is $X_v$, and the last layer generates embedding vectors $h_v^L$ for $\forall v \in V$, which can be used in downstream tasks, $L$ represents the total layer. In this paper, we focus on the node classification task. A softmax layer is employed on the final embedding vectors $h_v^L$ to get the class probability $C_v$ of node $v$.
\subsection{Problem Definition}
The goal of this paper is to preserve the adaptive privacy of the graph nodes from importance-grain in the training step of GNN through differential privacy. Different from previous work in \cite{daigavane2021node}, we aim to propose an epoch-independent method considering node importance from topology without losing much utility of GNN. Refer to the definition in \cite{sajadmanesh2022gap}, we first define the notion of Node-level adjacent graph in this paper as follows:\par
\begin{definition}[Node-level adjacent graph]
    Graphs $\mathcal{G}$ and $\mathcal{G}'$ are node-level adjacent if at most one node's feature vector is different. Without loss of generality, let $\mathcal{G}$ can be obtained by altering a node in $\mathcal{G}'$.
\end{definition}
Then we define the $\epsilon$-Node-level differential privacy as:\par
\begin{definition}[$\epsilon$-Node-level differential privacy]\label{node-level-DP}
    Let $\mathcal{G}$ and $\mathcal{G}'$ be two node-level adjacent graph datasets, given $\epsilon > 0$, the random algorithm $\mathcal{A}$ is $\epsilon$-Node-level differential privacy if for any set of outputs $S \in \text{Range}(\mathcal{A})$, satisfies:
    \begin{equation}
        \text{Pr}[\mathcal{A}(\mathcal{G}) \in S] \leq e^{\epsilon} \; \text{Pr}[\mathcal{A}(\mathcal{G}') \in S].
    \end{equation}
\end{definition}
Based on Definition \ref{node-level-DP}, the global graph sensitivity can be defined as:\par
\begin{definition}[Global graph sensitivity]
    The global graph sensitivity of function $f$ on two node-level adjacent graphs $\mathcal{G}$ and $\mathcal{G}'$ is:
    \begin{equation}
        \Delta_{g\mathcal{G}} = \text{max} ||f(\mathcal{G}) - f(\mathcal{G}')||_1
    \end{equation}
\end{definition}

\section{Proposed Method:NAP-GNN}
\begin{figure}[t]
    \centering
    \includegraphics[width=7.5cm]{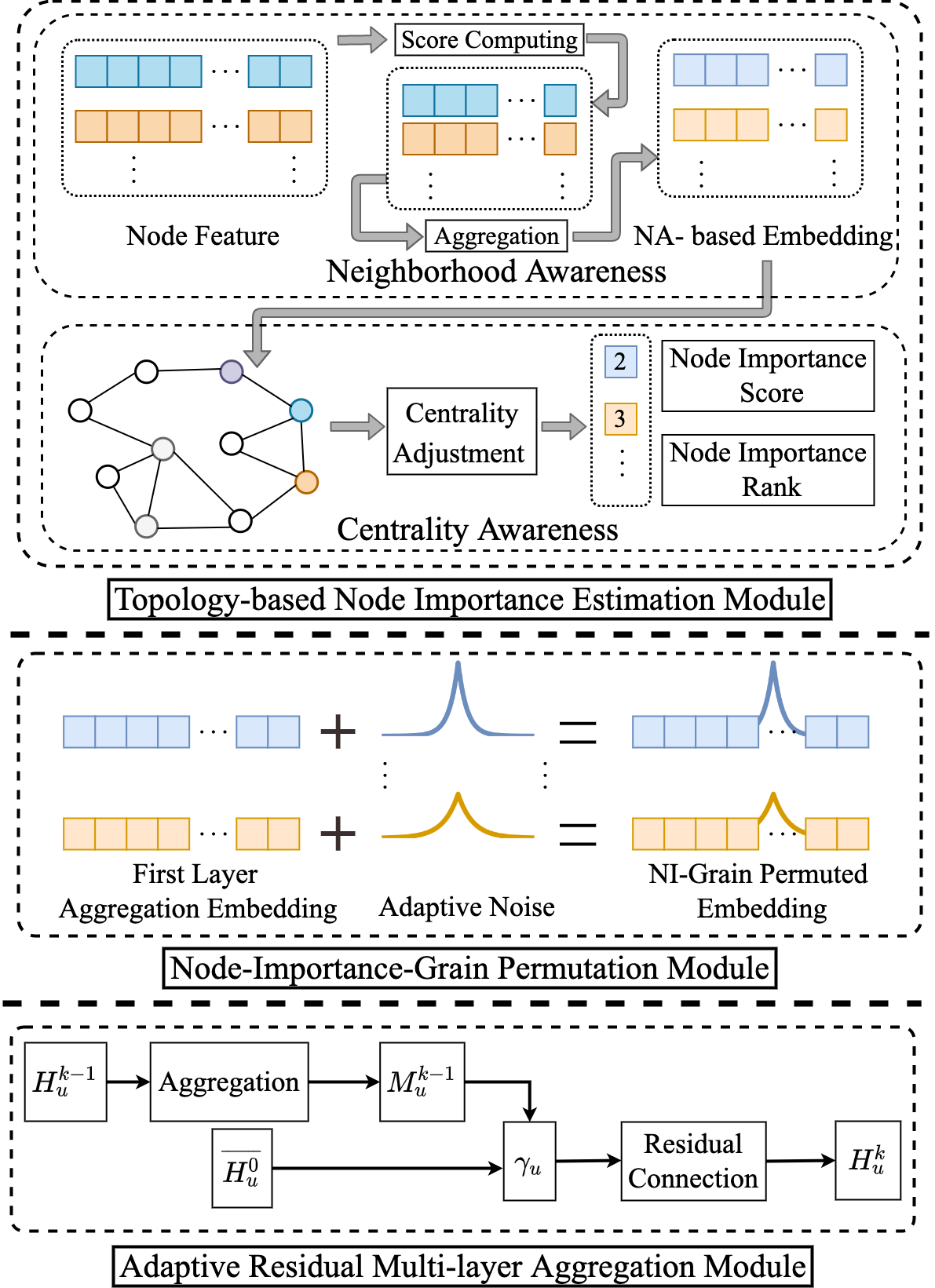}
    \caption{Overview of NAP-GNN’s architecture}
    \label{overview}
\end{figure}
In this section, we present our proposed \textbf{N}ode-Importance-Grained \textbf{A}daptive Differential \textbf{P}rivate \textbf{G}raph \textbf{N}eural \textbf{N}etwork (NAP-GNN). First, the overall flow of NAP-GNN is provided. Then we show the algorithm design and key components of NAP-GNN in detail. 

\subsection{Overall Flow of NAP-GNN}
Figure \ref{overview} shows the overall flow of NAP-GNN, which consists of three components: \textit{Topology-based Node Importance Estimation Module}, \textit{Node-Importance-Grained Permutation Module} and \textit{Adaptive Residual Multi-Layer Aggregation Module}. The key mechanism of NAP-GNN is to preserve the graph data privacy considering node importance, which is achieved in a fine-grain permutation module. We use the Laplacian mechanism to add noise on the aggregation vector in a customized way, motivated by the fact that altering a feature vector can be viewed as adding or mincing extra noise.\par

The propagation layer is essential to graph neural network utility. Compared to fully propagating all neighbors embedding, we propose to use an adaptive conversion between noised aggregation embedding and residual connection, letting different nodes employ customized propagation layers.
\subsection{Topology-based Node Importance Estimation}
To estimate topology-based node importance in a given graph, we propose a TNIE method considering the \textit{Neighborhood Awareness} and \textit{Centrality Awareness} motivated by \cite{park2019estimating}. Neighborhood Awareness refers to the idea that the importance of a node is influenced by its neighboring nodes, so the neighboring nodes' importance is a good representation of the current node's importance. Centrality Awareness means that nodes with higher centrality are more important than nodes with lower centrality. \par
The effect of neighboring nodes includes two aspects: node original feature and connectivity. To model this relationship, TNIE first uses a score computing network to map the original node features to a $r$-dimensional space:
\begin{equation}
    f_u^0 = ScoreComputing(\textbf{x}_\textbf{u}; u \in V),
\end{equation}
where $u$ is a node in $\mathcal{G}$. Score computing network can be any neural network that takes node original feature as input and output encoded feature. In this paper, we use Multi-Layer Perception (MLP). Then through intermediate embedding propagation, TNIE realizes weighted aggregation from node $u$ and its neighbors for the $t$-th layer ($t = 1, 2,... T$):
\begin{equation}
    f_u^t = \sum_{v \in N(u)} \frac{1}{\sqrt{D_u}} \frac{1}{\sqrt{D_v}} f_v^{t-1},
\end{equation}
where
$D_u$ denotes degree of $u$. After $T$ layers aggregation, $\forall u \in V$ gets neighborhood awareness-based embedding $f_u^T$.\par
The degree is a common proxy for node centrality. The degree of node has an impact on node importance estimation. Instead of initial degree $log(D_u + \alpha)$ of node $u$, where $\alpha$ is a small parameter, TNIE adopts a shifting degree $\lambda(log(D_u + \alpha)) + \phi)$ to allow the possible discrepancy between degree and importance rank, where $\lambda$ and $\phi$ are learnable parameters. Then the shifting degree is used to adjust the neighborhood awareness-based embedding $h_u^T$ of node $u$ from centrality awareness consideration, which is as follows:
\begin{equation}
    s(u)^* = \sigma((\lambda(log(D_u + \alpha)) + \phi)\cdot f_u^T),
\end{equation}
where $\sigma$ is a non-linear activation function, $s(u)^*$ denotes estimated importance score.\par
To estimate unknown nodes' importance score, TNIE uses mean squared error between the given importance score $\mathcal{NI}(u)$ and estimated score $s(u)^*$ for node $u \in V_t \subset V$ to train the TNIE model. The loss function is:
\begin{equation}
    \frac{1}{|V_t|}\sum_{u \in V_t}\big(s(u)^* - \mathcal{NI}(u) \big)^2
\end{equation}

\subsection{Node-Importance-Grained Adaptive Laplace Mechanism}
The goal of the Node-Importance-Grained Permutation Module is to privately release nodes' first aggregation embedding using adaptive noise proportion to sensitivity and node importance. Motivated by the fact that perturbing a node $u$'s edges in the graph can be seen as changing neighborhood aggregation of $u$'s adjacent nodes $\forall {v} \in N(u)$, we realize Node-Importance-Grained Permutation Module by adding appropriate noise on the first layer's aggregation function. Specially, we use the sum aggregation function as the first layer, which is equivalent to the multiplication of the adjacent matrix and the input row-normalized feature. The permutation process can be presented as follows:
\begin{equation}\label{sum-function}
    \overline{\mathcal{H}^0}(A, V, X) = \{\overline{H^0_u}\}_{u \in V} \; s.t. \; \overline{H^0_u} = \sum_{j=1}^{|N|} a_{uj} \textbf{x}_j + Lap(\frac{\Delta \mathcal{H}^0}{\epsilon_u}),
\end{equation}
where $H^0_u = \sum_{j=1}^{|N|} a_{uj} \textbf{x}_j$ denotes the sum aggregation process of node $u$, $a_{uj} \in A$, $\textbf{x}_j$ is the  row-normalized feature of node $j$, $Lap(\frac{\Delta \mathcal{H}^0}{\epsilon_u})$ is the perturbed noise, $\Delta \mathcal{H}^0$ denotes the sensitivity of aggregation function, $\epsilon_u$ is the privacy budget. \par
\begin{lemma}
    Let $\mathcal{G}=\{A,V,X\}$ and $\mathcal{G}='=\{A',V',X'\}$ be two adjacent graph dataset. 
    Then the global graph sensitivity of first sum aggregation layer $\Delta \mathcal{H}^0 \leq 2D_{max}$, where $D_{max}$ is maximum node degree. 
\end{lemma}
\begin{proof}
    Assume adjacent graph dataset $\mathcal{G}$ and $\mathcal{G}'$ differs in node $k$. Then we have:
    \begin{equation}
        \begin{aligned}
            \Delta \mathcal{H}^0 
            & = \max_{u \in V}||\sum\mathcal{H}^0(A, V, X) - \sum\mathcal{H}^0(A', V', X')||_1 \\
            & = \max \sum_{u \in V} ||\sum_{j=1}^{|N|}(a_{uj}\textbf{x}_j - a_{uj}'\textbf{x}_j')||_1
        \end{aligned}
    \end{equation}
    $A$ and $A'$ are two adjacent matrix. Without loss of generality, in $\mathcal{G}'$, we assume that node $k$ is removed from $\mathcal{G}$. Therefore, for nodes $i$ and $j$, we have $a_{ij}' = 0, \; if \; i = k \;or \;j = k$, otherwise $a_{ij}' = a_{ij}$.
    Then we can get the following inequality:
    \begin{equation}
        \begin{aligned}
            \Delta \mathcal{H}^0 
            & = ||\sum_{j = 1}^{|N|}a_{kj}\textbf{x}_j + \sum_{i=1}^{|N|}a_{ik}\textbf{x}_k||_1 \leq \sum_{j = 1}^{|N|}a_{kj} + \sum_{i=1}^{|N|}a_{ik}\\
            &\leq D_{out_k} + D_{in_k} \leq 2D_{max},
        \end{aligned}
    \end{equation}
    which concludes the proof.
\end{proof}
Denote the \textit{Node Importance} of node $u$ as $\mathcal{NI}(u)$, the rank of importance as $\mathcal{R\_NI}(u)$.  $\mathcal{C}(u)$ represents the privacy protection level of node $u$. For node $m$ and $n$, if $\mathcal{NI}(m) > \mathcal{NI}(n)$, then $\mathcal{C}(m) > \mathcal{C}(n)$. Let $\epsilon_{\mathcal{A}}$ denotes the total privacy cost of first aggregation layer, then the assigned privacy budget $\epsilon_{u}$ of node $u$ is proportional to $\mathcal{C}(u)$, which is $\epsilon_{u} = \epsilon_{\mathcal{A}} \cdot \beta_u$, where $\beta_u$ is a weight coefficient that quantifies the impact of node importance in the privacy protection level.\par
Many real-world graph datasets follow power-law distribution.
When node degree distribution is extremely imbalance, the maximum node degree is obviously higher than most nodes, Laplace-based noise generation mechanism may yield high noise.
To tackle this challenge, we propose to leverage the potential wasted privacy budget generated by node degree difference of adjacent graph nodes.\par
For arbitrary node $u \in V$, it receives potential reusable privacy budget from neighboring node $k \in N(u)$. Total reusable privacy ratio of $k$ is $D_{max} - D_k$, for each neighbor node of $k$, the assigned ratio $r(u,k)$ is
\begin{equation}
    r(u,k) = \frac{\mathcal{R\_NI}(u
    )}{\sum_{j \in N(k)}\mathcal{R\_NI}(j)},
\end{equation}
which is proportion to node importance rank. Then the weight coefficient of $u$ is the minimum of all reusable budget ratio:
\begin{equation}
    \beta_u = \min_{k \in N(u)}\{\frac{\mathcal{R\_NI}(u
    )}{\sum_{j \in N(k)}\mathcal{R\_NI}(j)}(D_{max} - D_k) + 1, \frac{D_{max}}{D_k}\}.
\end{equation}

\subsection{Adaptive Residual Multi-Layer Aggregation}
The Adaptive Residual Multi-Layer Aggregation Module takes the node-importance-grained perturbed aggregation vectors as input and outputs the embedding vectors, which are generated through aggregating the multi-hops neighbors embedding with adaptive residual. Compared with equally receiving all neighbors embedding, the adaptive residual aggregation module allows each node to learn from different neighbors with different weights.\par 
To keep edges and node labels private, we perturb them before the message passing and updating steps. Edge Randomization \cite{wu2022linkteller} is selected as the edge pre-processing method because it can preserve the adjacency matrix's sparse structure. To reduce the impact of adding or removing edges on the graph, we propose a degree-preserving Edge Randomization method, in which an unbiased sample method is used at the end with the sampling probability of:
\begin{equation}
    p_u^{sample} = \frac{2D_u}{D_u + N - Ns + D_us}
\end{equation}
for node $u$, where $D_u$ is the degree of $u$ before permutation, $s$ is a parameter of Edge Randomization satisfying $s \geq \frac{2}{e^{\epsilon_B}+1}$, $\epsilon_B$ is the privacy budget. The expectation of sampled node degree $\mathbb{E}(\overline{D_u'})$ equals $\mathbb{E}(D_u)$ for $\forall u \in V$.
We exploit Random Response \cite{kairouz2016discrete} to encode node $i$'s labels for it outperforms other oracles in low dimensions. The transformation distribution is:
\begin{equation}\label{random response}
    p(y_i'|y_i) = 
    \left\{
        \begin{aligned}
            &\frac{e^{\epsilon_C}}{e^{\epsilon_C}+M-1}, if\; y_i'=y_i \\
            &\frac{1}{e^{\epsilon_C}+M-1}, \; otherwise,\\
        \end{aligned}
    \right.
\end{equation}
where $\epsilon_C$ is the privacy budget, $M$ is class number.\par
For the aggregation process, motivated by \cite{liu2021graph}, a node-wise adaptive embedding aggregation and the residual connection are applied to balance the smoothness of different perturbed vectors between neighbors and the node itself. The Adaptive Message Passing (AMP) process includes three steps: feature aggregation, residual weight computing, and linear combination, which can be presented as:
\begin{equation}\label{feature aggregation}
    \textbf{AMP}(u, k, 1): M_u^{k-1} = \Tilde{A}H_u^{k-1},
\end{equation}
\begin{equation}\label{weight computing}
    \textbf{AMP}(u, k, 2): \gamma_u = \max(1-\tau \cdot \frac{1}{||M_u^{k-1} - \overline{H_u^0}||_2}, 0),
\end{equation}
\begin{equation}\label{linear combination}
    \textbf{AMP}(u, k, 3): H_u^{k} = (1-\gamma_u)\overline{H_u^0} + \gamma_uM_u^{k-1}.
\end{equation}\par

Equation (\ref{feature aggregation}) shows the feature aggregation step for node $u$ at layer $k \in \{1, 2,... K\}$, where $\Tilde{A} = \hat{D}^{-\frac{1}{2}}\hat{A}\hat{D}^{-\frac{1}{2}}$, $\hat{A} = \overline{A'} + I$ is the sampled noise adjacent matrix with self loop, $D$ is the degree matrix, $H_u^0 = \overline{H_u^0}$. Residual weight $\beta_u$ of node $u$ is computed through equation (\ref{weight computing}) depending on the deviation of adaptive perturbed embedding $\overline{H_u^0}$ and feature aggregation of neighbors $M_u^k$, $\tau$ is a parameter that controls the smoothing. Finally, the aggregation embedding of $u$ at layer $k+1$ is the linear combination of $\overline{H_u^0}$ and $M_u^k$ weighted by $\beta_u$ in (\ref{linear combination}).

\section{Theoretical Analysis}
In this section, we give the theoretical analysis of NAP-GNN. The complete process of NAP-GNN is shown in Algorithm \ref{algorithm}.
\begin{theorem} \label{sum-dp}
    Algorithm 1 preserves $\epsilon_A$-DP in the computation of first differential private aggregation layer $\mathcal{H}^0$.
\end{theorem}
\begin{proof}
    All nodes' aggregation embedding in $\mathcal{G}$ are perturbed, therefore we have:
    \begin{equation}
        Pr(\overline{\mathcal{H}^0}(A, V, X)) = \prod_{i=1}^N exp(\frac{\epsilon_i}{\Delta \mathcal{H}^0}||\sum_{j=1}^N a_{ij}\textbf{x}_j - \overline{H^0_i}||_1).
    \end{equation}
    $\Delta \mathcal{H}^0$ is set to $2D_{max}$ in Algorithm \ref{algorithm}. Assume adjacent graph dataset $\mathcal{G}$ and $\mathcal{G}'$ differs in node $k$. We have:
    \begin{equation} \label{dp}
        \begin{aligned}
        & \frac{Pr(\overline{\mathcal{H}^0}(A, V, X))}{Pr(\overline{\mathcal{H}^0}(A', V', X'))}\\
        & \leq \prod_{i=1}^N exp(\frac{\epsilon_i}{\Delta \mathcal{H}^0} ||\sum_{j=1}^N a_{ij}\textbf{x}_j-\sum_{j=1}^N a'_{ij}\textbf{x}'_j||_1)\\
        & = exp(\sum_{i=1}^N \sum_{j=1}^N \frac{\epsilon_i}{\Delta \mathcal{H}^0} || a_{ij}\textbf{x}_j- a'_{ij}\textbf{x}'_j||_1)\\
        & = exp(\sum_{j \in N(k)} a_{kj}\frac{\epsilon_k}{\Delta\mathcal{H}^0} + \sum_{i \in N(k)} a_{ik}\frac{\epsilon_i}{\Delta \mathcal{H}^0}).
        \end{aligned}
    \end{equation}
    Let $f(k) = \sum_{j\in N(k)}a_{kj}\epsilon_k + \sum_{i \in N(k)}a_{ik}\epsilon_i$. Then we have:
    \begin{equation}\label{tmp2}
        \begin{aligned}
            &f(k) = D_k \epsilon_k + \sum_{i\in N(k)}\epsilon_i = \epsilon_A (D_k \beta_k + \sum_{i \in N(k)} \beta_i)\\
            & \leq \epsilon_A (D_k \beta_k + \sum_{i \in N(k)}(\frac{\mathcal{R\_NI}(u)}{\sum_{j \in N(k)}\mathcal{R\_NI}(j)}(D_{max} - D_k + 1)) \\
            & \leq \epsilon_A(D_k \frac{D_{max}}{D_k} + D_{max} - D_k + D_k) \leq 2 \epsilon_A D_{max}.
        \end{aligned}
    \end{equation}
    Substitute equation (\ref{tmp2}) into (\ref{dp}), we can get:
    \begin{equation}
        \frac{Pr(\overline{\mathcal{H}^0}(A, V, X))}{Pr(\overline{\mathcal{H}^0}(A', V', X'))} \leq exp(\frac{2 \epsilon_A D_{max}}{\Delta \mathcal{H}^0}) = exp(\epsilon_A).
    \end{equation}
    Theorem \ref{sum-dp} is proved.
\end{proof}

Since the first aggregation function is linear, the generated embedding is unbiased, as follows:

\begin{proposition} \label{unbiased}
    The sum aggregator function defined in (\ref{sum-function}) for the first layer is unbiased.
\end{proposition}

Now we analyze the unbiased property of degree-preserving edge sampling:
\begin{proposition}
    Let $\overline{A}$ denote the noised adjacent matrix obtained by Edge Randomization, $D_u$ is the degree of node $u$ in the original graph, $\overline{D_u'}$ denotes the degree after sampling. Then $\mathbb{E}(\overline{D_u'}) = \mathbb{E}(D_u)$.
\end{proposition}
\begin{proof}
In the original adjacent matrix $A$, there are $D_u$ 1s and $(N-D_u)$ 0s for node $u$. The expectation of $u$'s degree after Edge Randomization is:
\begin{equation}
\begin{aligned}
    \mathbb{E}(\overline{D_u}) 
    & = D_us + D_u(1-s) + 0.5(N-D_u)(1-s)\\
    & = 0.5D_u + 0.5N - 0.5Ns + 0.5D_us,
\end{aligned}
\end{equation}
where $s$ is the Bernoulli sample probability \cite{wu2022linkteller}. Then we can have $\mathbb{E}(\overline{D_u'}) = \mathbb{E}(\overline{D_u}) * p_u^{sample} = \mathbb{E}(D_u)$.
\end{proof}

\begin{algorithm}[t]
    \caption{Node-Importance-Grained Adaptive Differential Private Graph Neural Network}
    \label{algorithm}
    \textbf{Input}: Graph $\mathcal{G}=\{V, E, A, X, Y\}$, the node set of known node importance $V_t$, depth $T$ and $K$, maximum node degree $D_{max}$, privacy budget $\epsilon_A$, $\epsilon_B$ and $\epsilon_C$, balance coefficient $\tau$\\
    \textbf{Output}: Node label predictions
    \begin{algorithmic}[1] 
        \STATE \textbf{Estimate node importance from topology perspective} \\ 
        $\forall \; u \in \mathcal{N}$, get $f_u=ScoreComputing(x_u)$\\
        \FOR{$t \in \{1,2,3...,T\}$ }
        \STATE $f_u^t = \sum_{v \in N(u)} \frac{1}{\sqrt{D_u}} \frac{1}{\sqrt{D_v}} f_v^{t-1}$
        \ENDFOR
        \STATE Get neighborhood awareness-based embedding $f^T_u$ 
        \STATE Estimated score: $s(u)^* = \sigma((\lambda(log(D_u + \alpha)) + \phi)\cdot f_u^T)$\\
        \STATE Train:
        $loss = \frac{1}{|V_t|}\sum_{u \in V_t}\big(s(u)^* - \mathcal{NI}(u) \big)^2$\\
        \STATE Get $\mathcal{NI}(i)$ and $\mathcal{R}\_\mathcal{NI}(i)$ of $i \in V - V_t$
        \STATE \textbf{Inject adaptive Laplace noise into differential Aggregation layer $\mathcal{H}^0$ and noise to $A$ and $Y$} \\
        \STATE $\Delta(\mathcal{H}^0) = 2D_{max}$\\
        \STATE $\forall i \in N$, compute weight coefficient $\beta_i$, keeping if $\mathcal{R\_NI}(i) > \mathcal{R\_NI}(j)$, then $\beta_i > \beta_j$
        \FOR{ $H^0(A,V,X)_i \in \mathcal{H}^0$}
        \STATE $\epsilon_i = \beta_i * \epsilon_{A}$ \\
        \STATE $H^0(A,V,X)_i = \sum_{j \in N(i)} a_{ij}\textbf{x}_j  + Lap(\frac{\Delta(\mathcal{H}^0)}{\epsilon_i})$
        \ENDFOR
        \STATE Perturb adjacent matrix $A$ under $\epsilon_B$ and sample edges with degree-preserving 
        \STATE Perturb known node label under $\epsilon_C$ 
        \STATE \textbf{Adaptive residual aggregate multi-hop neighbors}\\
        \FOR{$k \in \{1,2,3...,K\}$ }
        \STATE first step in AMP, get feature aggregation:\\
        $\textbf{AMP}(u, k, 1): M_u^{k-1} = \Tilde{A}H_u^{k-1}$
        \STATE second step in AMP, compute residual weight:\\
        $\textbf{AMP}(u, k, 2): \gamma_u = \max(1-\tau \cdot \frac{1}{||M_u^{k-1} - \overline{H_u^0}||_2}, 0)$
        \STATE third step in AMP, linear combine the aggregated feature and noised vector:\\
        $\textbf{AMP}(u, k, 3): H_u^{k} = (1-\gamma_u)\overline{H_u^0} + \gamma_uM_u^{k-1}$
        \ENDFOR
        \RETURN label predictions $\{y_i\}$ for unknown label nodes
    \end{algorithmic}
\end{algorithm}
\begin{table*}
    \centering
    \begin{tabular}{c|c c|c c|c c|cc}
    \hline
        Dataset & \multicolumn{2}{c|}{Cora} & \multicolumn{2}{c|}{Citeseer} & \multicolumn{2}{c|}{Lastfm}  & \multicolumn{2}{c}{Facebook} \\ \hline
        Privacy Level & Low & High & Low & High & Low & High & Low & High \\ \hline
        MLP-DP & 38.3$\pm$0.92 & 30.7$\pm$1.5 & 36.5$\pm$1.59 & 26.4$\pm$2.52 & 41.7$\pm$1.39 & 26.1$\pm$1.34 & 53.2$\pm$0.35 & 43.4$\pm$0.51 \\ 
        DP-GNN & 64.2$\pm$1.17 & 55.7$\pm$2.80 & 54.5$\pm$0.85 & 42.1$\pm$4.25 & 72.1$\pm$0.43 & 63.3$\pm$1.73 & 73.4$\pm$0.69 & 59.8$\pm$0.93 \\ \hline
        GCN-DP$_{equa}$ & 75.4$\pm$0.45 & 62.6$\pm$0.65 & 63.1$\pm$0.21 & 42.1$\pm$0.62 & 77.2$\pm$0.49 & 57.7$\pm$0.47 & 82.5$\pm$0.31 & 61.1$\pm$0.17 \\ 
        GCN-DP$_{adap}$ & 77.2$\pm$0.42 & 58.3$\pm$0.35 & 64.1$\pm$0.39 & 42.2$\pm$0.39 & 79.5$\pm$0.58 & 59.2$\pm$0.43 & 86.1$\pm$0.21 & 62.0$\pm$0.15 \\ 
        SAGE-DP$_{equa}$ & 63.8$\pm$0.33 & 44.1$\pm$0.63 & 46.1$\pm$1.19 & 31.3$\pm$0.45 & 72.9$\pm$0.24 & 50.1$\pm$0.36 & 85.7$\pm$1.31 & 51.3$\pm$1.18 \\ 
        SAGE-DP$_{adap}$ & 67.2$\pm$0.27 & 44.3$\pm$0.36 & 58.9$\pm$0.50 & 31.5$\pm$0.48 & 73.3$\pm$0.50 & 51.6$\pm$0.69 & 86.0$\pm$1.44 & 50.1$\pm$0.91 \\
        Ours$_{equa}$ & 83.1$\pm$0.11 & 63.4$\pm$0.26 & 64.5$\pm$0.21 & 42.4$\pm$0.17 & 84.5$\pm$0.10 & 65.3$\pm$0.15 & 89.8$\pm$0.07 & 65.5$\pm$1.10 \\
        \textbf{Ours} & \textbf{83.5$\pm$1.84} & \textbf{64.9$\pm$1.75} & \textbf{65.0$\pm$0.16} & \textbf{44.6$\pm$0.26} & \textbf{86.3$\pm$0.84 }& \textbf{65.9$\pm$0.15} & \textbf{91.3$\pm$0.64} & \textbf{67.8$\pm$0.17} \\ \hline
        Ours$_{deg}$ & 79.6$\pm$0.18 & 64.1$\pm$0.44 & 64.9$\pm$1.12 & 44.2$\pm$0.28 & 85.7$\pm$0.19 & 63.4$\pm$0.14 & 91.1$\pm$0.03 & 64.8$\pm$0.37 \\
        Ours$_{-edgeSam}$ & 81.5$\pm$0.15 & 64.5$\pm$1.30 & 63.8$\pm$1.14 & 43.9$\pm$2.07 & 85.4$\pm$1.87 & 64.9$\pm$0.14 & 90.0$\pm$1.07 & 67.1$\pm$0.09 \\ \hline
    \end{tabular}
    \caption{Performance of Compared Methods on Cora, Citeseer, Lastfm and Facebook. Test accuracy(\%) over two privacy level (for MLP-DP, set $\epsilon=10$ for high privacy cost and $\epsilon=20$ for low cost; for other models, set $\epsilon=7$ for high cost and $\epsilon=12$ for low cost, 
    ) is reported. `equa`, `adap`, `deg` and `-edgeSam` is short for `equal`, `adaptive`, `degree` and `without edge sample`, respectively.}
    \label{performance table}
\end{table*}

\begin{theorem}
    Algorithm 1 preserves $(\epsilon_A+\epsilon_B+\epsilon_C)$-DP. 
\end{theorem}
\begin{proof}
    From theorem \ref{sum-dp}, we know that the first aggregation layer satisfies $\epsilon_A$-DP. 
    Adjacent matrix perturbation guarantees $\epsilon_B$-DP for $\epsilon_B \geq (\ln{\frac{2}{s}-1}), s\in (0,1]$ \cite{wu2022linkteller}. Random response mechanism on node label guarantees $\epsilon_C$-DP under equation (\ref{random response}). \par
    Node-level DP preserves all information of one node, including feature, edges and label. Differential aggregation layer $\mathcal{H}^0$ process node feature and edge privately. The following Adaptive Residual Multi-Layer Aggregation Module do not expose node features and edges for it only post-processing the noised aggregation embedding and perturbed adjacent matrix without access private features and links. Training phase also guarantees DP because node label is perturbed former. The total privacy cost is $(\epsilon_A+\epsilon_B+\epsilon_C)$-DP.
\end{proof}
\section{Experimental Evaluation}
In this section, we provide experiments on the proposed method and evaluate it under different parameter settings.\par
\subsection{Experiment Settings}
\subsubsection{Datasets}
Four publicly available datasets are used, Cora, Citeseer \cite{yang2016revisiting}, Lastfm \cite{rozemberczki2020characteristic} and Facebook \cite{rozemberczki2021multi}. Cora and Citeseer are typical citation networks, representing two sparse graphs. Lastfm and Facebook are social networks, representing large scale and dense graphs compared with Cora and Citeseer. Detailed information of datasets is shown in Table \ref{dataset-table}. \par
\begin{table}[h]
    \centering
    \begin{tabular}{c c c c c }
    \toprule
         Datasets & Nodes & Edges & Features & Avg\_Deg \\
    \midrule
         Cora & 2708 & 5278 & 1433 & 3.89\\
         Citeseer & 3327 & 4552 & 3703 & 3.73 \\
         Lastfm & 7083 & 25814 & 7842 & 8.28 \\
         Facebook & 22470 & 170912 & 4717 & 15.21 \\
    \bottomrule
    \end{tabular}
    \caption{Detailed Statistic of Used Datasets}
    \label{dataset-table}
\end{table}
\subsubsection{Compared Methods}
We compare with the following methods: DP-GNN \cite{daigavane2021node} that apply DP-Adam on 1-layer GNN. We also compare with MLP-DP, for MLP does not use graph topology. It is trained with DP-Adam to provide node-level protection. 
What's more, we compare two types of variants by privacy cost distribution methods: equal and adaptive. Two state-of-art GNN architectures, GCN and GraphSAGE, are used to substitute the proposed adaptive residual multi-layer aggregation module. We compare the performance of node importance estimation between node degree and TNIE through substitute $\mathcal{NI}(i)$ with $D_i$. Degree-preserving edge sample is removed from NAP-GNN as competitors.
\subsubsection{Training and Evaluation Settings} 
We randomly split the nodes in the dataset into training, test, and validation sets in the ratio of 50\%, 25\%, and 25\%. Since the original Lastfm is highly imbalanced, we choose the top 10 classes that have most samples. The compared methods, GCN and GraphSAGE, are two convolutional layers with hidden dimension 128, RELU activation function, and a dropout layer with a probability of 0.2. We use the same noise scale on all noise-added perturbation modules. All models are trained based on Adam optimizer \cite{kingma2014adam} with a maximum epoch of 1500. The initial value of the learning rate is 1e-3, and the decay mechanism is used with patience of 20 and a decay rate of 0.5. We measure the model performance by training 10 consecutive rounds on the test set and taking the average value with a 95\% confidence interval under bootstrapping with 2000 samples. All experiments are implemented by using Pytorch and PyTorch-Geometric (PyG). \par
\subsection{Results and Analysis}
\subsubsection{Evaluation on Adaptive Budget over Different $\epsilon$}

\begin{figure}[h]
    \centering
    \includegraphics[width=8cm]{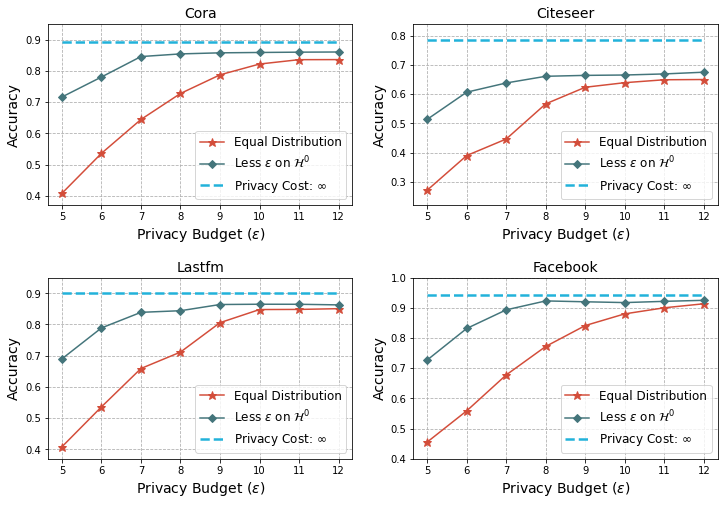}
    \caption{Evaluation on Different Privacy Budget.}
    \label{different_e}
\end{figure}

\begin{figure*}[h]
    \centering
    \includegraphics[width=17.5cm]{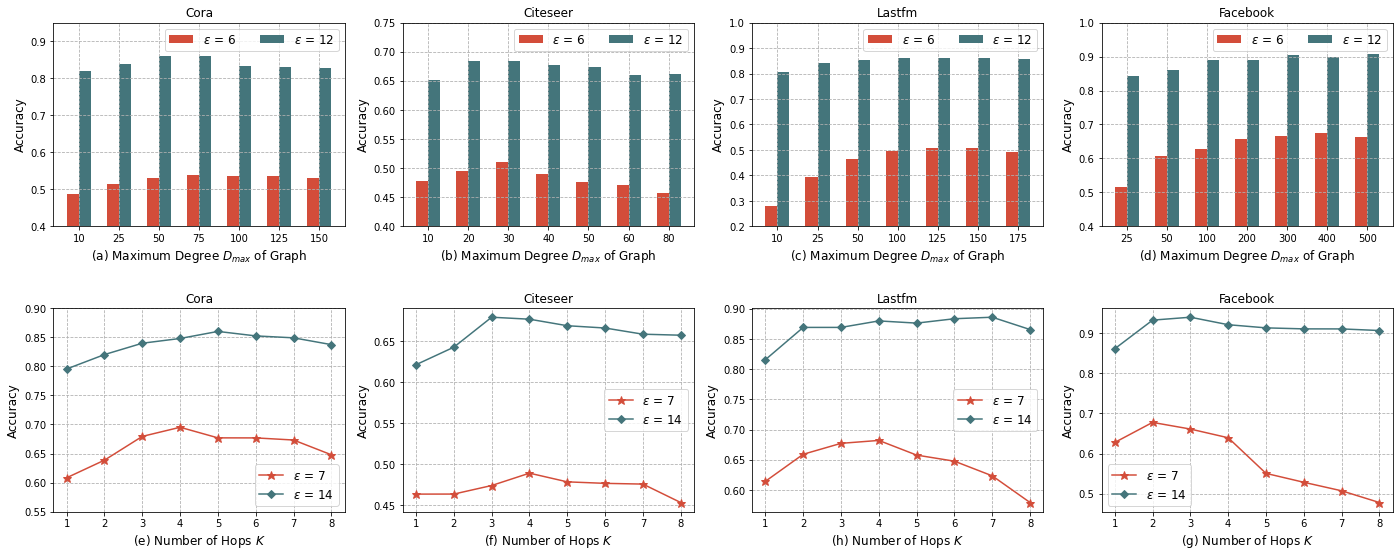}
    \caption{Evaluation on Maximum Degree $D_{max}$ of Graph (a-d) and Hop $K$ of Multi-Adaptive-Layer (e-g).}
    \label{degree-hop}
\end{figure*}

We first compare the difference in accuracy between the proposed NAP-GNN and competitors under different privacy budgets. The results are shown in Table \ref{performance table} and Figure \ref{different_e}. It can be observed that the adaptive DP mechanism achieves higher accuracy than equally distributing the total privacy budget under the same method in most cases.
The performance gap between the two differential privacy methods is influenced by the average degree graph. Under the same global privacy budget, we also compare different noise scale ratios between differential private aggregation layer $\mathcal{H}^0$, the degree-preserving adjacent matrix $\overline{A}$ and noised label $\overline{Y}$. With noise on $\overline{A}$ be constant, we add more noise on $\mathcal{H}^0$ (decline $\epsilon_A$), which is less $\epsilon$ on $\mathcal{H}^0$. Compared with the equal distribution method, the accuracy improvement of accuracy represents adaptive DP on $\mathcal{H}^0$ can mitigate the noise waste problem caused by node over-protecting. The scheme that assigns less $\epsilon$ on $\mathcal{H}^0$ converges faster. As $\epsilon$ increases (noise decreases), the model's utility tends to be the same for the two compared schemes.
\par

\subsubsection{Evaluation on Node Importance Estimation}
We investigate the effectiveness of the importance estimation module. We compare the case where TNIE is used as usual and the case where node degree is regarded as the importance score. The accuracy gap of method \textit{Ours$_{deg}$} in line 10 and \textit{Ours} in line 9 of Table \ref{performance table} demonstrates that the performance of the model with TNIE is better. This is because node degree only reflects the neighbor number and does not consider interactions between the adjacent nodes and the node feature.

\subsubsection{Evaluation on Degree-Preserving Edge Sampling}
We compare the usual NAP-GNN with the case that edge sampling is removed and directly inputs noised adjacent matrix to the next module. The result is shown in line 11 in Table \ref{performance table}. We can see that in all cases, the accuracy of NAP-GNN is higher with edge sampling than without it.

\subsubsection{Evaluation on Maximum Degree $D_{max}$ of Graph}

We analyze the impact of the maximum node degree of the graph on model performance. The performance variation of the adaptive differential privacy mechanism is shown in Figure \ref{degree-hop}(a)-(d). The accuracy of the model increases to a peak and then decreases as $D_{max}$ increases. The maximum degree of the sub-graphs taken at the peak differs on data with different degree average value. When $D_{max}$ increases, the neighbor information that nodes can aggregate grows, but at the same time, more noise is received, thus affecting the accuracy of the model. Meanwhile, the Laplace mechanism is affected by data sensitivity. When $D_{max}$ increases, the data sensitivity increases, and thus the noise added to each node on average increases, reducing the utility of the model.\par
\subsubsection{Evaluation on Multi-Adaptive-Layer $K$}
We investigate the effect of different hops on the model performance. The results are shown in Figure \ref{degree-hop}(e)-(g). As can be seen, both NAP-GNN and its competitors can aggregate more information from neighbors when $K$ increases. However, there is a trade-off between $K$ and accuracy. When the value of $K$ increases, the accuracy of both NAP-GNN and its competitors increases first,  then reaches a peak and decreases. This is because when more layers of neighbor information are aggregated, the noise data collected from the neighbors also increase, affecting the behavior of the model. Besides, when $\epsilon$ is small, the noise scale of the joined data is larger, and the model needs higher $K$ to reach the peak.
\section{Conclusion}
In this paper, we propose a novel adaptive differential private graph neural network from node-importance-grain. We use adaptive aggregation perturbation, where the node-importance-grained Laplace mechanism is applied to the first aggregation function, and adaptive residual multi-layer aggregation, where embedding vectors are generated through multi-hops with the adaptive residual connection. Moreover, we present a new node importance estimation method named TNIE from graph topology considering neighborhood and centrality awareness. Experimental results over real-world graph datasets show that our NAP-GNN can achieve better privacy and accuracy trade-off and outperforms existing methods. In the future, we would like to investigate the rationale behind the trade-off between graph properties and noise scale from theoretical and algorithmic views, which would improve model utility and promote differentially private graph neural networks in practice.  \par

\clearpage
\bibliographystyle{named}
\bibliography{ijcai23}

\begin{thebibliography}{}

\bibitem[\protect\citeauthoryear{Chen \bgroup \em et al.\egroup
  }{2020}]{chen2020phishing}
Weili Chen, Xiongfeng Guo, Zhiguang Chen, Zibin Zheng, and Yutong Lu.
\newblock Phishing scam detection on ethereum: Towards financial security for
  blockchain ecosystem.
\newblock In {\em IJCAI}, pages 4506--4512, 2020.

\bibitem[\protect\citeauthoryear{Clauset \bgroup \em et al.\egroup
  }{2009}]{clauset2009power}
Aaron Clauset, Cosma~Rohilla Shalizi, and Mark~EJ Newman.
\newblock Power-law distributions in empirical data.
\newblock {\em SIAM review}, 51(4):661--703, 2009.

\bibitem[\protect\citeauthoryear{Daigavane \bgroup \em et al.\egroup
  }{2021}]{daigavane2021node}
Ameya Daigavane, Gagan Madan, Aditya Sinha, Abhradeep~Guha Thakurta, Gaurav
  Aggarwal, and Prateek Jain.
\newblock Node-level differentially private graph neural networks.
\newblock {\em arXiv preprint arXiv:2111.15521}, 2021.

\bibitem[\protect\citeauthoryear{Dwork \bgroup \em et al.\egroup
  }{2014}]{dwork2014algorithmic}
Cynthia Dwork, Aaron Roth, et~al.
\newblock The algorithmic foundations of differential privacy.
\newblock {\em Foundations and Trends{\textregistered} in Theoretical Computer
  Science}, 9(3--4):211--407, 2014.

\bibitem[\protect\citeauthoryear{Gasteiger \bgroup \em et al.\egroup
  }{2021}]{gasteiger2021directional}
Johannes Gasteiger, Chandan Yeshwanth, and Stephan G{\"u}nnemann.
\newblock Directional message passing on molecular graphs via synthetic
  coordinates.
\newblock {\em Advances in Neural Information Processing Systems},
  34:15421--15433, 2021.

\bibitem[\protect\citeauthoryear{Guan \bgroup \em et al.\egroup
  }{2022}]{guan2022robognn}
Sheng Guan, Hanchao Ma, and Yinghui Wu.
\newblock Robognn: Robustifying node classification under link perturbation.
\newblock In {\em Proceedings of the Thirty-First International Joint
  Conference on Artificial Intelligence (IJCAI-22)}, 2022.

\bibitem[\protect\citeauthoryear{Jiang \bgroup \em et al.\egroup
  }{2020}]{jiang2020federated}
Meng Jiang, Taeho Jung, Ryan Karl, and Tong Zhao.
\newblock Federated dynamic gnn with secure aggregation.
\newblock {\em arXiv preprint arXiv:2009.07351}, 2020.

\bibitem[\protect\citeauthoryear{Kairouz \bgroup \em et al.\egroup
  }{2016}]{kairouz2016discrete}
Peter Kairouz, Keith Bonawitz, and Daniel Ramage.
\newblock Discrete distribution estimation under local privacy.
\newblock In {\em International Conference on Machine Learning}, pages
  2436--2444. PMLR, 2016.

\bibitem[\protect\citeauthoryear{Kingma and Ba}{2014}]{kingma2014adam}
Diederik~P Kingma and Jimmy Ba.
\newblock Adam: A method for stochastic optimization.
\newblock {\em arXiv preprint arXiv:1412.6980}, 2014.

\bibitem[\protect\citeauthoryear{Liu \bgroup \em et al.\egroup
  }{2021}]{liu2021graph}
Xiaorui Liu, Jiayuan Ding, Wei Jin, Han Xu, Yao Ma, Zitao Liu, and Jiliang
  Tang.
\newblock Graph neural networks with adaptive residual.
\newblock {\em Advances in Neural Information Processing Systems},
  34:9720--9733, 2021.

\bibitem[\protect\citeauthoryear{Olatunji \bgroup \em et al.\egroup
  }{2021}]{olatunji2021releasing}
Iyiola~E Olatunji, Thorben Funke, and Megha Khosla.
\newblock Releasing graph neural networks with differential privacy guarantees.
\newblock {\em arXiv preprint arXiv:2109.08907}, 2021.

\bibitem[\protect\citeauthoryear{Park \bgroup \em et al.\egroup
  }{2019}]{park2019estimating}
Namyong Park, Andrey Kan, Xin~Luna Dong, Tong Zhao, and Christos Faloutsos.
\newblock Estimating node importance in knowledge graphs using graph neural
  networks.
\newblock In {\em Proceedings of the 25th ACM SIGKDD international conference
  on knowledge discovery \& data mining}, pages 596--606, 2019.

\bibitem[\protect\citeauthoryear{Rozemberczki and
  Sarkar}{2020}]{rozemberczki2020characteristic}
Benedek Rozemberczki and Rik Sarkar.
\newblock Characteristic functions on graphs: Birds of a feather, from
  statistical descriptors to parametric models.
\newblock In {\em Proceedings of the 29th ACM international conference on
  information \& knowledge management}, pages 1325--1334, 2020.

\bibitem[\protect\citeauthoryear{Rozemberczki \bgroup \em et al.\egroup
  }{2021}]{rozemberczki2021multi}
Benedek Rozemberczki, Carl Allen, and Rik Sarkar.
\newblock Multi-scale attributed node embedding.
\newblock {\em Journal of Complex Networks}, 9(2):cnab014, 2021.

\bibitem[\protect\citeauthoryear{Sajadmanesh and
  Gatica-Perez}{2021}]{sajadmanesh2021locally}
Sina Sajadmanesh and Daniel Gatica-Perez.
\newblock Locally private graph neural networks.
\newblock In {\em Proceedings of the 2021 ACM SIGSAC Conference on Computer and
  Communications Security}, pages 2130--2145, 2021.

\bibitem[\protect\citeauthoryear{Sajadmanesh \bgroup \em et al.\egroup
  }{2022}]{sajadmanesh2022gap}
Sina Sajadmanesh, Ali~Shahin Shamsabadi, Aur{\'e}lien Bellet, and Daniel
  Gatica-Perez.
\newblock Gap: Differentially private graph neural networks with aggregation
  perturbation.
\newblock {\em arXiv preprint arXiv:2203.00949}, 2022.

\bibitem[\protect\citeauthoryear{Wang \bgroup \em et al.\egroup
  }{2021}]{wang2021hierarchical}
Heyuan Wang, Shun Li, Tengjiao Wang, and Jiayi Zheng.
\newblock Hierarchical adaptive temporal-relational modeling for stock trend
  prediction.
\newblock In {\em IJCAI}, pages 3691--3698, 2021.

\bibitem[\protect\citeauthoryear{Wu \bgroup \em et al.\egroup
  }{2022}]{wu2022linkteller}
Fan Wu, Yunhui Long, Ce~Zhang, and Bo~Li.
\newblock Linkteller: Recovering private edges from graph neural networks via
  influence analysis.
\newblock In {\em 2022 IEEE Symposium on Security and Privacy (SP)}, pages
  2005--2024. IEEE, 2022.

\bibitem[\protect\citeauthoryear{Yang \bgroup \em et al.\egroup
  }{2016}]{yang2016revisiting}
Zhilin Yang, William Cohen, and Ruslan Salakhudinov.
\newblock Revisiting semi-supervised learning with graph embeddings.
\newblock In {\em International conference on machine learning}, pages 40--48.
  PMLR, 2016.

\bibitem[\protect\citeauthoryear{Yang \bgroup \em et al.\egroup
  }{2021a}]{yang2021heterogeneous}
Liang Yang, Fan Wu, Zichen Zheng, Bingxin Niu, Junhua Gu, Chuan Wang, Xiaochun
  Cao, and Yuanfang Guo.
\newblock Heterogeneous graph information bottleneck.
\newblock In {\em IJCAI}, pages 1638--1645, 2021.

\bibitem[\protect\citeauthoryear{Yang \bgroup \em et al.\egroup
  }{2021b}]{yang2021consisrec}
Liangwei Yang, Zhiwei Liu, Yingtong Dou, Jing Ma, and Philip~S Yu.
\newblock Consisrec: Enhancing gnn for social recommendation via consistent
  neighbor aggregation.
\newblock In {\em Proceedings of the 44th international ACM SIGIR conference on
  Research and development in information retrieval}, pages 2141--2145, 2021.

\bibitem[\protect\citeauthoryear{Zhao \bgroup \em et al.\egroup
  }{2021}]{zhao2021csgnn}
Chengshuai Zhao, Shuai Liu, Feng Huang, Shichao Liu, and Wen Zhang.
\newblock Csgnn: Contrastive self-supervised graph neural network for molecular
  interaction prediction.
\newblock In {\em IJCAI}, pages 3756--3763, 2021.

\bibitem[\protect\citeauthoryear{Zhou \bgroup \em et al.\egroup
  }{2020}]{zhou2020privacy}
Jun Zhou, Chaochao Chen, Longfei Zheng, Xiaolin Zheng, Bingzhe Wu, Ziqi Liu,
  and Li~Wang.
\newblock Privacy-preserving graph neural network for node classification.
\newblock {\em arXiv preprint arXiv:2005.11903}, 2020.

\end{thebibliography}

\end{document}